\documentclass[journal]{IEEEtran}

\usepackage{graphicx}
\usepackage{amsmath}
\usepackage{amssymb}
\usepackage{amsthm}
\usepackage{multirow}
\usepackage{algorithm}
\usepackage{algorithmic}
\usepackage{xcolor}
\usepackage{warpcol}
\usepackage{bm}
\usepackage{colortbl}
\usepackage{makecell}
\usepackage{wrapfig}
\usepackage{overpic}
\usepackage{etoolbox}
\usepackage{xstring}

\graphicspath{{./Imgs/}}
\DeclareGraphicsExtensions{.pdf,.jpg,.png}

\newtheorem{myTheorem}{Theorem} 
\newtheorem{myLemma}{Lemma}

\newcommand{\figref}[1]{Fig.~\ref{#1}}
\newcommand{\tabref}[1]{Table~\ref{#1}}
\newcommand{\secref}[1]{Section~\ref{#1}}

\newcommand{\equref}[1]{Eq. (\ref{#1})}
\newcommand{\myPara}[1]{\vspace{.12in}\noindent\textbf{#1}}
\def\ie{\emph{i.e.}}
\def\eg{\emph{e.g.}}
\def\vs{\emph{vs.~}}
\def\etc{\emph{etc}}

\def\sArt{{state-of-the-art~}}

\newcommand{\first}[1]{{\textbf{\textcolor{red}{#1}}}}
\newcommand{\second}[1]{{\textbf{\textcolor{green}{#1}}}}
\newcommand{\third}[1]{{\textbf{\textcolor{blue}{#1}}}}

\definecolor{mycolor}{rgb}{.9,.9,.9}

\let\mybibitem\bibitem
\renewcommand{\bibitem}[1]{%
    \IfSubStr{#1}{-force}
        {\color{cyan}\mybibitem{#1}}
        {\color{black}\mybibitem{#1}}%
}

\begin{document}
\title{DNA: Deeply-supervised Nonlinear Aggregation for Salient Object Detection}

\author{Yun Liu,
        Ming-Ming Cheng,
        Xin-Yu Zhang,
        Guang-Yu Nie,
        and~Meng Wang
\thanks{Manuscript received April 19, 2005; revised August 26, 2015.
Major Project for New Generation of AI under Grant No. 2018AAA0100400, 
NSFC (NO. 61620106008),
S\&T innovation project from Chinese Ministry of Education,
and Tianjin Natural Science Foundation for Distinguished Young Scholars (NO. 17JCJQJC43700).}
\thanks{Y. Liu, M.-M. Cheng, and X.-Y. Zhang are with Nankai University. M.-M. Cheng (cmm@nankai.edu.cn) is the corresponding author.}
\thanks{G.-Y. Nie is with Beijing Institute of Technology.}
\thanks{M. Wang is with Hefei University of Technology.}}

\markboth{IEEE Transactions on Cybernetics}%
{Liu \MakeLowercase{\textit{et al.}}: DNA: Deeply-supervised Nonlinear Aggregation for Salient Object Detection}

\maketitle

\begin{abstract}
Recent progress on salient object detection mainly aims 
at exploiting how to effectively integrate multi-scale 
convolutional features in convolutional neural networks (CNNs).
Many popular methods impose deep supervision to 
perform side-output predictions that are linearly aggregated 
for final saliency prediction.
In this paper, we theoretically and experimentally demonstrate 
that linear aggregation of side-output predictions is suboptimal, 
and it only makes limited use of the side-output information 
obtained by deep supervision.
To solve this problem, we propose Deeply-supervised Nonlinear 
Aggregation (DNA) for better leveraging the complementary 
information of various side-outputs.
Compared with existing methods, it i) aggregates side-output 
features rather than predictions, and ii) adopts nonlinear 
instead of linear transformations.  
Experiments demonstrate that DNA can successfully break through 
the bottleneck of current linear approaches.
Specifically, the proposed saliency detector, a modified U-Net 
architecture with DNA, performs favorably against 
\sArt methods on various datasets and evaluation metrics 
without bells and whistles.
\end{abstract}

\begin{IEEEkeywords}
Salient object detection, saliency detection, deeply-supervised 
nonlinear aggregation.
\end{IEEEkeywords}

\IEEEpeerreviewmaketitle

\section{Introduction}
\IEEEPARstart{S}{alient} object detection, also known as saliency 
detection, aims at simulating the human vision system to detect 
the most conspicuous and eye-attracting objects or regions in  
natural images \cite{achanta2009frequency,cheng2015global,BorjiCVM2019}.
The progress in saliency detection has been beneficial to 
a wide range of vision applications, including 
image retrieval \cite{gao2013visual,Sketch2Photo},
visual tracking \cite{mahadevan2009saliency},
scene classification \cite{ren2014region},
content-aware image/video processing \cite{zund2013content,TIP20_SP_NPR},
thumbnail generation \cite{wang2016stereoscopic},
video object segmentation \cite{wang2018saliency},
and weakly supervised learning \cite{wei2017object,wei2018revisiting}.
Although numerous models have been presented 
\cite{li2018contour,chen2018reverse,zhang2018progressive,liu2018picanet,islam2018revisiting,he2017delving,wang2019new,yan2020new,li2020depthwise,chen2020embedding}
and significant improvement has been made, it still remains
an open problem to accurately detect complete salient objects 
in static images, especially in complicated scenarios.

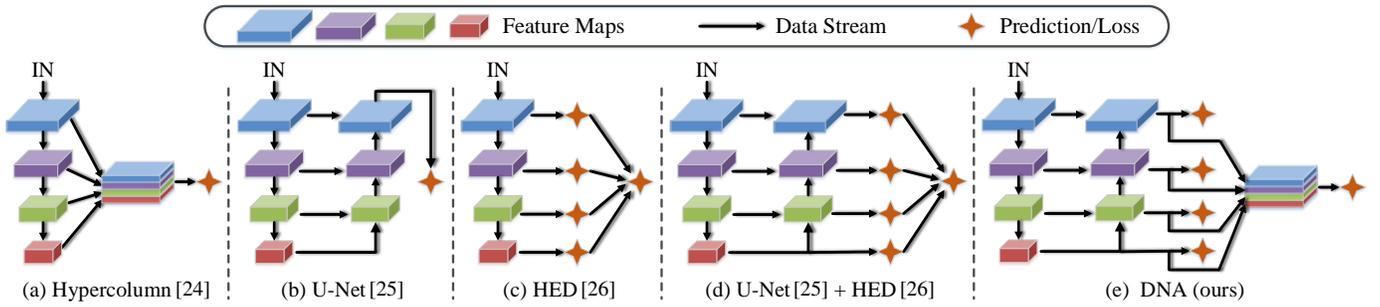
\begin{figure*}[!tb]
\centering
\begin{overpic}[width=\linewidth]{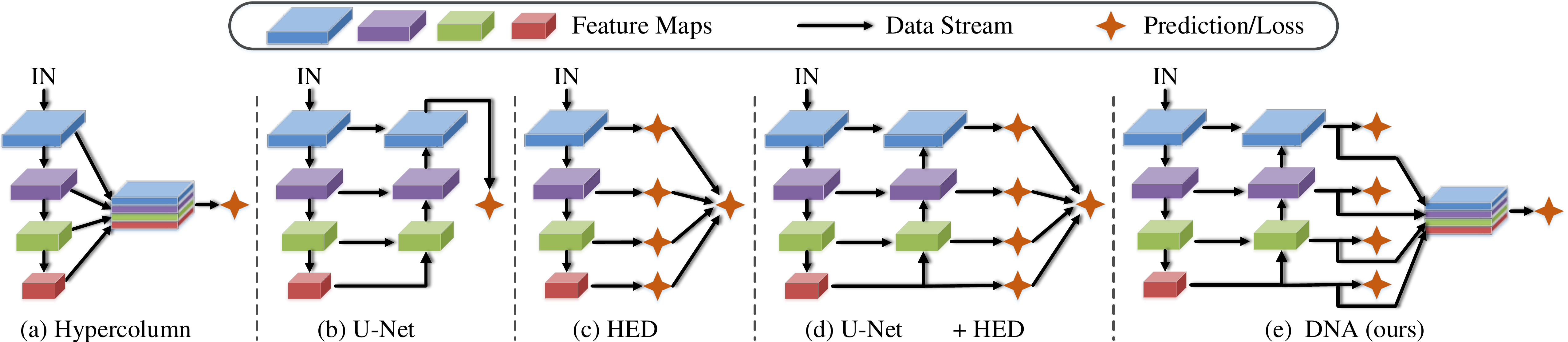}
    \put(12.5, 0.4){\footnotesize\cite{hariharan2015hypercolumns}}
    \put(26.7, 0.4){\footnotesize\cite{ronneberger2015u}}
    \put(42.2, 0.4){\footnotesize\cite{xie2017holistically}}
    \put(57.8, 0.4){\footnotesize\cite{ronneberger2015u}}
    \put(65.6, 0.4){\footnotesize\cite{xie2017holistically}}
\end{overpic}
\caption{Illustration of different multi-scale deep learning
architectures.
Note that (c)-(e) use deep supervision to produce side-outputs,
but (c) and (d) linearly aggregate side-output predictions, while
the proposed DNA (e) adopts nonlinear aggregation onto side-output
features.
} \label{fig:frame_comp}
\end{figure*}

Conventional saliency detection methods 
\cite{cheng2015global,jiang2013salient,tong2015salient}
usually design hand-crafted low-level features and heuristic priors, 
which are difficult to represent semantic objects and scenes.
Recent advances on saliency detection mainly benefit from 
\textit{convolutional neural networks} (CNNs) 
\cite{zhu2017saliency,guo2017video,wang2017deep,wang2017video,Fan2020S4Net}.
On the one hand, CNNs naturally learn multi-scale and multi-level 
feature representations in each layer due to the increasingly 
larger receptive fields and downsampled (strided) scales \cite{pami20Res2net}.
On the other hand, salient object detection requires multi-scale 
learning because of the various object/scene scales in 
intra- and inter-images \cite{liu2020refinedbox,cheng2019bing}.
Therefore, current cutting-edge saliency detectors
\cite{chen2018reverse,zeng2018learning,wang2018salient,zhang2018bi,liu2018picanet,wang2018detect,zhang2017amulet,wang2017stagewise,hou2019deeply} 
mainly aim at designing complex network architectures to leverage 
multi-scale CNN features, \eg, the semantic meaningful information 
in the top sides and the complementary spatial details in the 
bottom sides.

Owing to the superiority of U-Net \cite{ronneberger2015u} 
(or FCN \cite{long2015fully}) and HED \cite{xie2017holistically} 
in multi-scale learning, many leading-edge saliency 
detectors add deep supervision onto U-Net networks 
\cite{zhang2018progressive,wang2018salient,liu2018picanet,islam2018revisiting,zhang2017amulet,liu2016dhsnet,he2017delving},
\cite{guan2018edge} (\figref{fig:frame_comp}(d)).
We note that these networks first predict multi-scale saliency 
maps using side-outputs.
The generated multi-scale side-output predictions are then linearly 
aggregated, \eg, via a pixel-wise convolution (\ie, $1 \times 1$ 
convolution), to obtain the final saliency prediction which can  
thus combine the advantages of all side-output predictions.
However, we theoretically and experimentally demonstrate that the 
\textbf{linear aggregation of side-output predictions} is suboptimal, 
and it makes limited use of the complementary multi-scale 
information implicated in side-output features.
We provide detailed proofs in \secref{sec:linear}.

Instead of linearly aggregating side-output predictions,
we propose a nonlinear side-output aggregation method.
Specifically, we concatenate the side-output features rather 
than side-output predictions and then apply nonlinear 
transformations to predict salient objects.
We also impose deep supervision to side-output features 
for better optimization in the training phase,
as illustrated in \figref{fig:frame_comp}(e).
In this way, the concatenated features can make better use 
of the multi-scale side-output features. 
We call the resulting method \textbf{Deeply-supervised Nonlinear 
Aggregation (DNA)}.
%
We apply DNA into a simply redesigned U-Net without 
bells and whistles.
The proposed network performs favorably against all previous \sArt 
salient object detectors with less parameters and faster speed.
Our contributions are twofold:
\begin{itemize}
\item We theoretically and experimentally analyze the natural 
limitation of traditional linear side-output aggregation which
can only make limited use of multi-scale side-output information.
\item We propose Deeply-supervised Nonlinear Aggregation (DNA) for 
side-output features, whose effectiveness has been proved by 
introducing it into a simple network with less parameters and faster speed. 
\end{itemize}

\section{Related Work} \label{sec:related}
Salient object detection is a very active research field due to 
its wide range of applications and challenging scenarios. 
Early \textit{heuristic saliency detection} methods extract 
hand-crafted low-level features and apply machine learning models 
to classify these features 
\cite{gong2015saliency,tu2016real,xia2017and}.
Some heuristic saliency priors are utilized to ensure the accuracy,
such as color contrast \cite{achanta2009frequency,cheng2015global},
center prior \cite{jiang2013submodular,jiang2013salient} 
and background prior \cite{yang2013saliency,zhu2014saliency}.
With vast successes achieved by deep CNNs in computer vision, 
CNN-based methods have been introduced to improve saliency detection
\cite{zhang2017supervision,zhang2018deep,tavakoli2017saliency,lang2016dual,li2016deepsaliency}.
\textit{Region-based saliency detection} 
\cite{zhao2015saliency,wang2015deep,li2015visual,lee2016deep,chen2016disc,wang2019new,yan2020new}
appeared in the early era of deep learning based saliency.
These approaches view each image patch as a basic processing 
unit to perform saliency detection.
More recently, \textit{CNN-based image-to-image saliency detection} 
\cite{chen2018reverse,zhang2018progressive,zeng2018learning,wang2018salient,zhang2018bi,liu2018picanet,wang2018detect,islam2018revisiting,zhang2017amulet,chen2017look,wang2017stagewise,liu2018deep,liu2018learning,hou2019deeply,li2016deep,luo2017non,liu2016dhsnet,li2020depthwise,chen2020embedding}
has dominated this field by viewing saliency detection as a pixel-wise 
regression task and performing image-to-image predictions.
Hence we mainly review CNN-based image-to-image saliency detection 
in the following.

Since saliency detection requires both high-level global information 
(existing in the top sides of CNNs) and low-level local details 
(existing in the bottom sides of CNNs), how to effectively fuse 
multi-level deep features is the main research direction 
\cite{chen2018reverse,zeng2018learning,wang2018salient,zhang2018bi,liu2018picanet,wang2018detect,zhang2017amulet,wang2017stagewise,hou2019deeply,li2016deep,luo2017non,liu2016dhsnet,wu2019cascaded,zhang2019capsal}.
There are too many studies to list here, but the general trend 
of recent network designs is to become more and more complicated. 
We continue our discussion by briefly categorizing multi-scale 
deep learning into four classes: \textit{hyper feature learning}, 
\textit{U-Net style}, \textit{HED style},
and \textit{U-Net + HED style}.
An overall illustration of them is shown in \figref{fig:frame_comp}.

\myPara{Hyper feature learning:}
Hyper feature learning \cite{hariharan2015hypercolumns,hu2018learning,liu2018deepseg} 
is the most intuitive way to learn multi-scale information,
as illustrated in \figref{fig:frame_comp}(a).
Examples of this structure for saliency include 
\cite{li2016deep,zeng2018learning,chen2017look,wang2017stagewise,liu2018deep,su2019selectivity,liu2018learning,zhao2019pyramid}.
These models concatenate/sum multi-scale deep features 
from multiple layers of backbone nets 
\cite{li2016deep,zeng2018learning} 
or branches of the multi-stream nets 
\cite{chen2017look,wang2017stagewise,liu2018deep}. 
The fused hyper features, called \textit{hypercolumn}, are then used 
for final predictions.

\myPara{U-Net style:}
It is widely accepted that the top layers of deep neural networks 
contain high-level semantic information, while the bottom 
layers learn low-level fine details.
Therefore, a reasonable revision of hyper feature learning 
is to progressively fuse deep features from upper layers to 
lower layers \cite{long2015fully,ronneberger2015u}, 
as shown in \figref{fig:frame_comp}(b). 
The top semantic features will combine with bottom low-level 
features to capture fine-grained details.
The feature fusion can be a simple element-wise summation 
\cite{long2015fully}, a simple feature map concatenation (U-Net) 
\cite{ronneberger2015u}, or complex designs based on them.
Many saliency detectors are of this type 
\cite{wang2016saliency,zhang2017learning,luo2017non,hu2017deep,wang2018detect,zhang2018bi,li2018contour,bruce2016deeper,qin2019basnet}.
Note that hyper feature learning and U-Net do not apply deep 
supervision, so they \textbf{do not have side-outputs}.

\myPara{HED style:}
HED-like networks \cite{xie2017holistically,liu2019richer,liu2018semantic} 
were first presented for edge detection.
Afterwards, similar ideas have been also introduced for saliency 
detection \cite{chen2018reverse,hou2019deeply}.
HED-like networks add deep supervision at the intermediate sides 
to obtain \textbf{side-output predictions}, and the final result is a linear 
combination of all side-output predictions (shown in 
\figref{fig:frame_comp}(c)). 
Unlike multi-scale feature fusion, HED performs multi-scale 
prediction fusion.

\myPara{U-Net + HED style:}
These methods combine the advantages of both U-Net and HED.
We outline this architectures in \figref{fig:frame_comp}(d).
Specifically, deep supervision is imposed at each of the convolution 
stage of U-Net decoder.
Many recent saliency models fall into this category 
\cite{zhang2018progressive,wang2018salient,liu2018picanet,islam2018revisiting,zhang2017amulet,qiu2020miniseg,liu2016dhsnet,he2017delving,wang2019salient,feng2019attentive,liu2019simple,wu2019mutual,chen2020embedding}, \cite{guan2018edge}.
They differ from each other by applying different fusion strategies. 
One notable similarity of these models is that the final prediction 
is produced by a linear aggregation of side-output predictions.
Hence the multi-scale learning is achieved \textbf{in two aspects}: 
i) the U-Net aggregates multi-level convolutional features 
from top layers to bottom layers in an encoder-decoder form;
ii) the multi-scale side-output predictions are further linearly 
aggregated for final prediction.
\textbf{Current research in this field mainly focuses on the first 
aspect}, and top-performing models have designed very 
complex feature fusion strategies for this
\cite{liu2018picanet,zhang2017amulet}.

A full literature review of salient object detection is beyond 
the scope of this paper. Please refer to 
\cite{bylinskii2018different,cong2018review,han2018advanced} 
for more comprehensive surveys.
In this paper, we focus on the second aspect of above 
\textit{U-Net + HED} multi-scale learning: the multi-scale 
side-output aggregation. 
We find that the upper bound of traditional linear side-output
prediction aggregation is limited to the side-output predictions.
Hence we propose DNA to aggregate side-output features in the 
nonlinear way, so that the aggregated hybrid features can make good 
use of the complementary multi-scale deep features.
A streamlined diagram of our proposed DNA can be seen 
in \figref{fig:frame_comp}(e).
We demonstrate DNA can achieve superior performance 
with a very simple U-Net.

\section{Revisiting Linear Side-output Aggregation} \label{sec:linear}
Deep supervision and corresponding linear side-output prediction 
aggregation have been demonstrated to be effective in many vision tasks 
\cite{xie2017holistically,liu2019richer,liu2018picanet,zhang2017amulet}.
This section analyzes the natural limitation of the linear side-output 
aggregation from both theoretical and experimental perspectives. 
To the best of our knowledge, this is a novel contribution.

Suppose a deeply-supervised network has $N$ side-output prediction 
maps $\{\mathcal{O}_1,\mathcal{O}_2,\cdots,\mathcal{O}_N\}$, 
all of which are supervised by ground-truth maps 
(\figref{fig:frame_comp}(c)(d)). 
Without loss of generality, we assume the linear side-output 
aggregation is a pixel-wise convolution, \ie, $1 \times 1$
convolution. 
Hence, current linear side-output aggregation can be 
written as 
\begin{equation} \label{equ:lin_sum}
\widehat{\mathcal{O}} = \sum_{i=1}^N \bm{w}_i \cdot \mathcal{O}_i,
\end{equation}
where weights $\bm{w}_i$ of pixel-wise convolution 
can be learned.
Note that we have $\bm{w}_i \geq 0$ here. 
Otherwise, $\mathcal{O}_i$ would have negative effect to $\widehat{\mathcal{O}}$, 
so it should be excluded in the aggregation. 
To obtain the output saliency probability map, a standard 
sigmoid function $\sigma(x) = \frac{1}{1+e^{-x}}$ should be applied 
to $\widehat{\mathcal{O}}$.
The aggregated probability map becomes 
\begin{equation} \label{equ:lin_sigmoid}
\widehat{\mathcal{P}} = \sigma(\widehat{\mathcal{O}}) = \sigma(\sum_{i=1}^N \bm{w}_i \cdot \mathcal{O}_i).
\end{equation}
Similarly, we can compute side-output probability maps 
$\{\mathcal{P}_1,\mathcal{P}_2,\cdots,\mathcal{P}_N\}$.

\begin{myTheorem} \label{math:w1}
If $\Vert\bm{w}\Vert_1=1$, the mean absolute error (MAE) of 
fused output $\widehat{\mathcal{P}}$ is limited by side-output predictions. 
\end{myTheorem}
\begin{proof}
If $\Vert\bm{w}\Vert_1=1$, it is natural to show  
\begin{equation}
\min(\mathcal{O}_i) \leq \sum_{i=1}^N \bm{w}_i \cdot \mathcal{O}_i \leq \max(\mathcal{O}_i),
\end{equation}
because $\bm{w}_i \geq 0$ as discussed above. 
Since the sigmoid function $\sigma(x)$ is monotonically increasing, 
we have
\begin{equation}
\min(\mathcal{P}_i) \leq \widehat{\mathcal{P}} \leq \max(\mathcal{P}_i).
\end{equation}
If a pixel $\bm{p}$ is positive, we have
$\text{MAE}(\widehat{\mathcal{P}})_{\bm{p}} = \vert 1-\widehat{\mathcal{P}}(\bm{p}) \vert = 1-\widehat{\mathcal{P}}(\bm{p})$ 
and $1-\max(\mathcal{P}_i)_{\bm{p}} \leq 1-\widehat{\mathcal{P}}(\bm{p}) \leq 1-\min(\mathcal{P}_i)_{\bm{p}}$, 
so that 
$\min(\text{MAE}(\mathcal{P}_i)_{\bm{p}}) \leq \text{MAE}(\widehat{\mathcal{P}})_{\bm{p}} \leq \max(\text{MAE}(\mathcal{P}_i)_{\bm{p}})$ holds.
If the pixel $\bm{p}$ is negative, we have
$\text{MAE}(\widehat{\mathcal{P}})_{\bm{p}} = \vert 0-\widehat{\mathcal{P}}(\bm{p}) \vert = \widehat{\mathcal{P}}(\bm{p})$ 
and $\min(\mathcal{P}_i)_{\bm{p}} \leq \widehat{\mathcal{P}}(\bm{p}) \leq \max(\mathcal{P}_i)_{\bm{p}}$, 
so that 
$\min(\text{MAE}(\mathcal{P}_i)_{\bm{p}}) \leq \text{MAE}(\widehat{\mathcal{P}})_{\bm{p}} \leq \max(\text{MAE}(\mathcal{P}_i)_{\bm{p}})$ holds.
%
%
Note that $\bm{w}$ usually only has $N$ ($N \leq 6$ in VGG16 
\cite{simonyan2014very} and ResNet \cite{he2016deep}) dimensions, 
so it is also difficult to make aforementioned left equality hold.
Hence traditional linear aggregation is limited in terms of MAE metric. 
However, what we expect is to break through the limitation by 
making full use of multi-scale information. 
\end{proof}

\begin{myLemma} \label{math:lemma}
If $\Vert\bm{w}\Vert_1 \neq 1$, traditional linear aggregation (as in 
\equref{equ:lin_sum} and \equref{equ:lin_sigmoid}) is equivalent to first 
applying an aggregation with $\Vert\bm{\tilde{w}}\Vert_1 = 1$ 
and then applying a monotonically increasing mapping.
\end{myLemma}
\begin{proof}
If $\Vert\bm{w}\Vert_1 \neq 1$, we set 
$\bm{w} = \bm{\tilde{w}} \cdot \Vert\bm{w}\Vert_1$, so we have 
$\Vert\bm{\tilde{w}}\Vert_1 = 1$ .
The computation of $\widehat{\mathcal{P}}$ becomes 
\begin{equation}
\widehat{\mathcal{P}} = \sigma(\Vert\bm{w}\Vert_1 \cdot \sum_{i=1}^N \bm{\tilde{w}}_i \cdot \mathcal{O}_i),
\end{equation}
in which $\sigma(\Vert\bm{w}\Vert_1 \cdot x)$ $(\Vert\bm{w}\Vert_1 > 0)$ 
is a monotonically increasing function in terms of $x$. 
\end{proof}

\begin{figure}
    \centering
    \includegraphics[width=.7\linewidth]{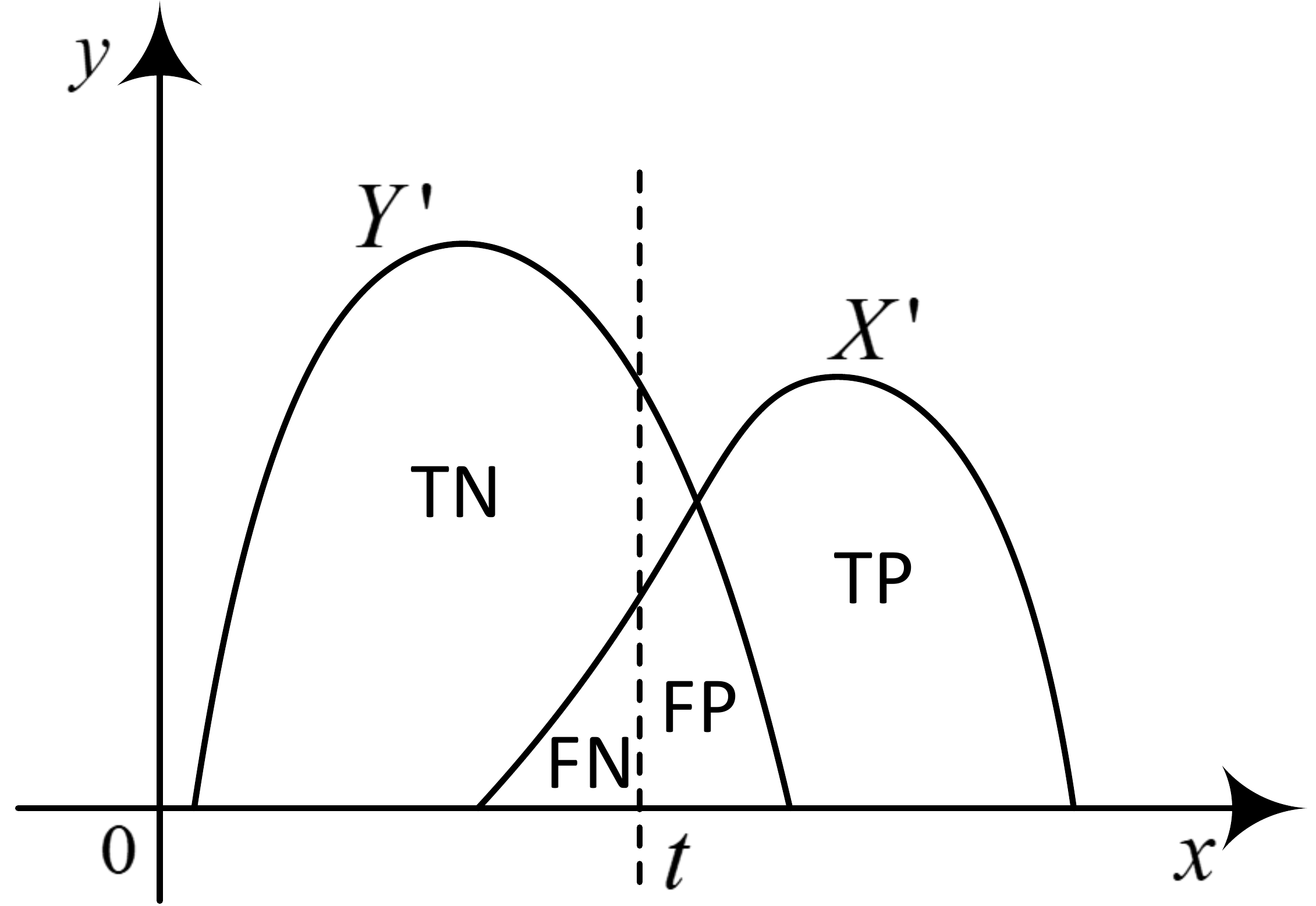} \\
    \caption{Probability ($x$ axis) \vs the density of $X'$ and $Y'$ 
    ($y$ axis). TN: true negative; FN: false negative; TP: true 
    positive; FP: false positive. 
    } \label{fig:math}
\end{figure}

\begin{myTheorem} \label{math:roc}
The monotonically increasing mapping of 
$\sigma(\Vert\bm{w}\Vert_1 \cdot x)$ $(\Vert\bm{w}\Vert_1 > 0)$ 
cannot change the ROC curve and AUC metric\footnote{AUC is the 
area under the ROC Curve.}.
\end{myTheorem}
\begin{proof}
Suppose the predicted scores of positive samples obey the 
distribution of $X \sim F(x)$, while the predicted scores of negative 
samples obey the distribution of $Y \sim G(x)$. We may assume $F$ and $G$ are continuous functions.
$\varphi(x) = \sigma(k \cdot x)$ $(k>0)$ is a variant of sigmoid function, 
so we have $\varphi: \mathbb{R} \to (0, 1)$ and $\varphi$ is a monotonically 
increasing function. 
Let $X' = \varphi(X)$ and $Y' = \varphi(Y)$ be two transformed 
distributions. 
It is easy to show 
\begin{equation}
\begin{aligned}
\mathbb{P}(X' \leq u) = \mathbb{P}(\varphi(X) \leq u) &= \mathbb{P}(X \leq \varphi^{-1}(u)) \\
&= F(\varphi^{-1}(u)),
\end{aligned}
\end{equation}
and thus we can obtain $X' \sim F(\varphi^{-1}(x))$ 
and $Y' \sim G(\varphi^{-1}(x))$.

\begin{table*}[!tb]
\centering
\renewcommand{\tabcolsep}{3.3mm}
\caption{Comparison between linear side-output prediction aggregation (\ie, 
lin) and nonlinear side-output feature aggregation (\ie, nonlin). 
The datasets and metrics will be introduced in \secref{sec:experiment_setup}.
The linear aggregation of HED \cite{xie2017holistically} and 
DSS \cite{hou2019deeply} follows the original papers, and their nonlinear 
aggregation replaces linear aggregation with the proposed DNA.
} \label{tab:lin_vs_nonlin}
\begin{tabular}{c|c||c|c|c|c|c|c|c|c|c|c} \Xhline{1.0pt}
    \multirow{2}*{Methods} & \multirow{2}*{Fusion} 
    & \multicolumn{2}{c|}{DUTS-TE} & \multicolumn{2}{c|}{ECSSD} 
    & \multicolumn{2}{c|}{HKU-IS} & \multicolumn{2}{c|}{DUT-O} 
    & \multicolumn{2}{c}{THUR15K}
    \\ \cline{3-12}
    & & $F_\beta$ & MAE & $F_\beta$ & MAE & $F_\beta$ & MAE 
    & $F_\beta$ & MAE & $F_\beta$ & MAE \\ \hline
    \multirow{2}*{HED \cite{xie2017holistically}}
    & linear & 0.796 & 0.079 & 0.892 & 0.065 & 0.893 & 0.052 
    & 0.726 & 0.100 & 0.757 & 0.099 \\ \cline{2-12}
    & nonlinear & \textbf{0.827} & \textbf{0.057} & \textbf{0.911} 
    & \textbf{0.053} & \textbf{0.912} & \textbf{0.039} & \textbf{0.752} 
    & \textbf{0.078} & \textbf{0.775} & \textbf{0.083} \\ \hline
    \multirow{2}*{DSS \cite{hou2019deeply}}
    & linear & 0.827 & 0.056 & 0.915 & 0.056 & 0.913 & 0.041 
    & 0.774 & 0.066 & 0.770 & 0.074 \\ \cline{2-12}
    & nonlinear & \textbf{0.833} & \textbf{0.055} & \textbf{0.918} 
    & \textbf{0.056} & \textbf{0.916} & \textbf{0.040} & \textbf{0.784} 
    & \textbf{0.060} & \textbf{0.773} & \textbf{0.072} \\ \hline
    \multirow{2}*{DNA} & linear & 0.844 & 0.048 & 0.921 & 0.050 
    & 0.917 & 0.034 & 0.765 & 0.066 & 0.785 & 0.071 \\ \cline{2-12} 
    & nonlinear & \textbf{0.865} & \textbf{0.044} & \textbf{0.935} 
    & \textbf{0.041} & \textbf{0.930} & \textbf{0.031} & \textbf{0.799} 
    & \textbf{0.056} & \textbf{0.793} & \textbf{0.069} \\ \Xhline{1.0pt}
\end{tabular}
\end{table*}

Let $t$ be a threshold, true positive rate (TPR) and false positive 
rate (FPR) can be computed as 
\begin{equation}
\begin{aligned}
\text{TPR} = \frac{\text{TP}}{\text{TP}+\text{FN}} = \mathbb{P}(X' > t) = 1 - F(\varphi^{-1}(t)), \\
\text{FPR} = \frac{\text{FP}}{\text{FP}+\text{TN}} = \mathbb{P}(Y' > t) = 1 - G(\varphi^{-1}(t)),
\end{aligned}
\end{equation}
as shown in \figref{fig:math}.
Hence we can denote the ROC curve as 
$\{(1 - F(\varphi^{-1}(t)), 1 - G(\varphi^{-1}(t))): t \in (0, 1)\}$.
It is easy to show that as $t$ goes from $0$ to $1$ continuously, $(1 - F(\varphi^{-1}(t)), 1 - G(\varphi^{-1}(t)))$ will change from $(1, 1)$ to $(0, 0)$ continuously and monotonically.  
It is also obvious to see that 
$\{(1 - F(\varphi^{-1}(t)), 1 - G(\varphi^{-1}(t)))\}$ 
and $\{(F(\varphi^{-1}(t)), G(\varphi^{-1}(t)))\}$ are symmetric about 
the point $(\frac{1}{2},\frac{1}{2})$.
Suppose the area under the 
curve $\{(1 - F(\varphi^{-1}(t)), 1 - G(\varphi^{-1}(t)))\}$ is $S_1$, and the area under 
the curve $\{(F(\varphi^{-1}(t)), G(\varphi^{-1}(t)))\}$ is $S_2$.
By symmetry, we have $S_1 + S_2 = 1$.

With the above conclusions, we can compute $S_2$ as 
\begin{equation}
\begin{aligned}
S_2 &= \int_0^1 G(\varphi^{-1}(t)) d F(\varphi^{-1}(t)) \\
&= \int_{-\infty}^{+\infty} G(x) d F(x).
\end{aligned}
\end{equation}
Therefore, $S_2$ is independent of the specific form of the 
function $\varphi(x)$, and $S_1 = 1-S_2$ is independent of 
$\varphi(x)$, too.
Moreover, as $t$ ranges in $(0, 1)$, thus $\varphi^{-1}(t)$ ranges in $\mathbb{R}$.
We have 
\begin{equation}
\begin{aligned}
&\{(1 - F(\varphi^{-1}(t)), 1 - G(\varphi^{-1}(t))): t \in (0, 1)\} \\
&=\{(1-F(x), 1-G(x)): x \in \mathbb{R}\},
\end{aligned}
\end{equation}
which is also independent of the form of $\varphi(x)$.
When $F(x)$ and $G(x)$ are discrete, the set 
$\{(1 - F(\varphi^{-1}(t)), 1 - G(\varphi^{-1}(t))): t \in (0, 1)\}$
is discrete but still independent of $\varphi(x)$.
Therefore, we can conclude that $\varphi(x)$ cannot change 
the ROC curve and AUC metric.
\end{proof}

Similar to the proof for Theorem \ref{math:w1}, we can easily 
demonstrate that the first step in Lemma \ref{math:lemma},
\ie, linear aggregation with $\Vert\bm{\tilde{w}}\Vert_1 = 1$, 
has limited MAE results. 
From Theorem \ref{math:roc}, we know the second step in 
Lemma \ref{math:lemma}, \ie, a monotonically increasing mapping,
cannot change the ROC curve and AUC value. 
Therefore, we can conclude that traditional linear aggregation 
with $\Vert\bm{w}\Vert_1 \neq 1$ has limited improvement. 
Combined with Theorem \ref{math:w1}, we can conclude that 
linear aggregation of side-outputs only has limited improvement.

Besides the theoretical proofs, we also perform experiments to 
compare linear aggregation versus nonlinear aggregation 
for salient object detection. 
To this end, we use the proposed nonlinear side-output feature 
aggregation (in \secref{sec:dna}) for nonlinear regression to 
evaluate two well-known models: HED \cite{xie2017holistically} 
and DSS \cite{hou2019deeply}, and the proposed DNA model.
The results are summarized in \tabref{tab:lin_vs_nonlin}.
We can see significant improvement from linear regression to 
nonlinear regression. 
Based on this observation, this paper aims at designing 
\textbf{a simple network with nonlinear side-output aggregation} 
for effective salient object detection.

\begin{figure}[!tb]
    \centering
    \includegraphics[width=\linewidth]{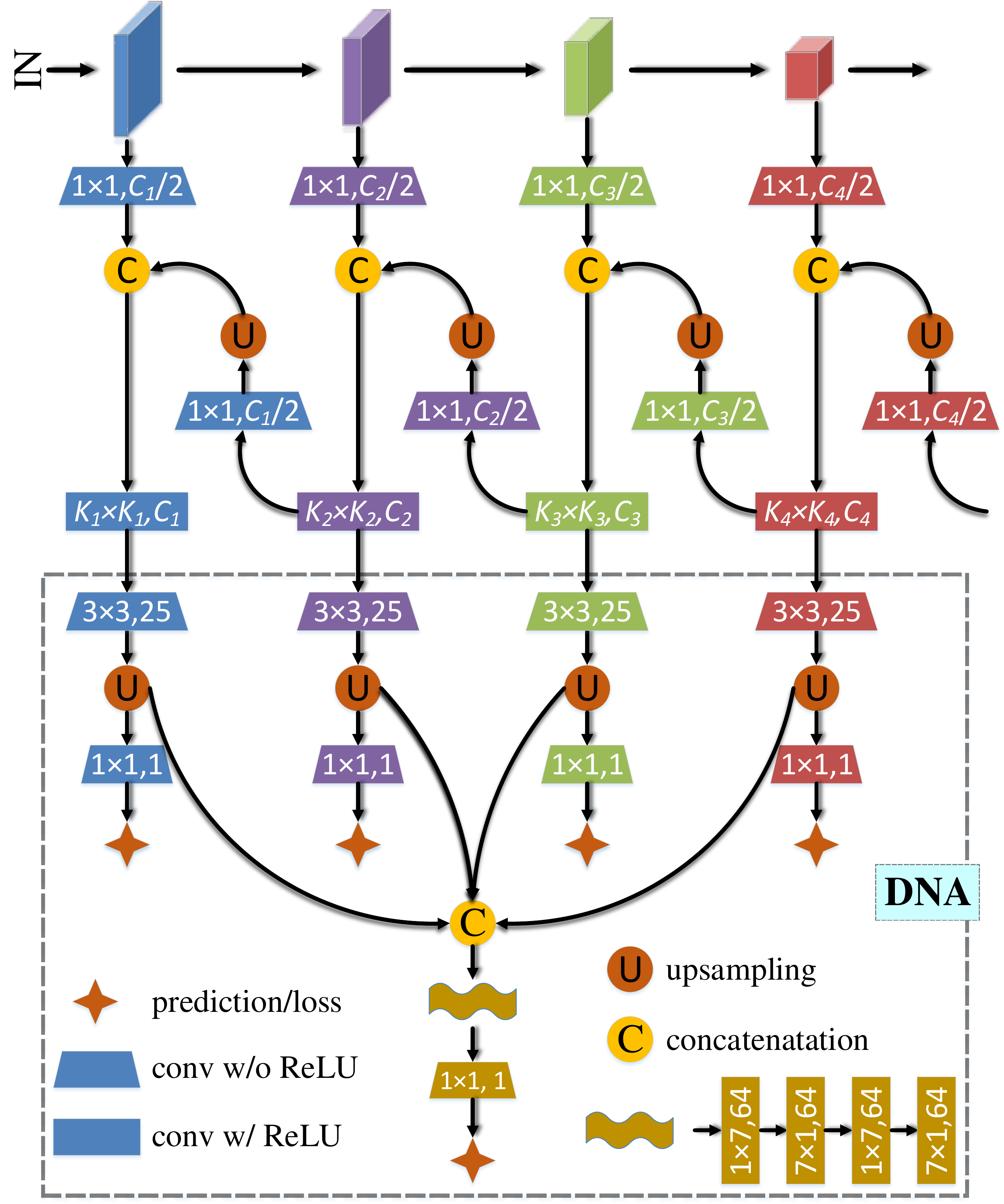} \\
    \caption{Network architecture. We only illustrate the first four 
    network sides in this figure, and the other two can be constructed 
    in the same way. The proposed DNA module is in the dotted box. 
    The parameters $K_i \times K_i$ and $C_i$ are introduced in the text.
    } \label{fig:framework}
\end{figure}

\section{Methodology}
In this section, we will elaborate our proposed framework 
for salient object detection. 
We first introduce our base network in \secref{sec:basenet}.
Then, we present the deeply-supervised nonlinear aggregation  
in \secref{sec:dna}.
An overall network architecture is illustrated 
in \figref{fig:framework}.

\subsection{Base Network} \label{sec:basenet}
\noindent\textbf{Backbone network.}
To tackle the salient object detection, we follow recent studies
\cite{chen2017look,wang2017stagewise,hou2019deeply}
to use fully convolutional networks.
Specifically, we use the well-known VGG16 network 
\cite{simonyan2014very} as our backbone network, whose final  
fully connected layers are removed to serve for image-to-image 
translation.
Salient object detection usually requires global information to 
locate the coarse positions of salient objects \cite{cheng2015global},
so enlarging the receptive field of the network would be helpful.
To this end, we keep the final pooling layer of VGG16 as in 
\cite{hou2019deeply} and replace the last two fully connected 
layers with two convolution layers, one of which has the kernel size 
of $3 \times 3$ with $C_6^{(1)}=192$ channels and another of which has 
the kernel size of $7 \times 7$ with $C_6^{(2)}=128$ channels.
Here, we use the $3 \times 3$ convolution layer to reduce the 
feature channels, because large kernel sizes (\eg, $7 \times 7$) 
lead to much more parameters.

There are five pooling layers in the backbone network.
They divide the convolution layers into six convolution 
blocks, which are denoted as $\{\mathcal{S}^1,
\mathcal{S}^2,\mathcal{S}^3,\mathcal{S}^4,\mathcal{S}^5, 
\mathcal{S}^6\}$ from bottom to top, respectively.
We consider $\mathcal{S}^6$ as the top valve that controls the 
overall contextual information that flows in the network.
The resolution of the feature maps in each convolution block is 
half of the preceding one.
Following \cite{hou2019deeply,xie2017holistically}, the 
side-output of each convolution block is connected from the last
layer of this block.

\myPara{Encoder-decoder network.}
Based on the backbone net, we build an encoder-decoder network
that can be seen in \figref{fig:framework}.
Concretely, we connect a $1 \times 1$ convolution layer to each of 
the convolution blocks $\mathcal{S}^6$ and $\mathcal{S}^5$ to 
adjust the number of channels (as shown in \tabref{tab:net_config}).
Then, we upsample the obtained feature maps from $\mathcal{S}^6$ 
by two. 
The upsampled feature maps and resulting feature maps from 
$\mathcal{S}^5$ are concatenated.
To fuse the concatenated feature maps, two sequential convolution 
layers are used to generate the decoder side
$\widetilde{\mathcal{S}}^5$.
The decoder sides 
$\{\widetilde{\mathcal{S}}^4,\widetilde{\mathcal{S}}^3,
\widetilde{\mathcal{S}}^2,\widetilde{\mathcal{S}}^1\}$ 
can be obtained in the same manner. 
For a clear presentation, we formulate the above process as follows  
\begin{equation}
\begin{aligned}
\widetilde{\mathcal{S}}^i \hspace{0.018\linewidth} &=\varphi({\rm Concat}
(\phi_1(\mathcal{S}^i),\phi_2(\widetilde{\mathcal{S}}^{i+1}))), \\
\phi_1(\cdot) &= {\rm Conv}(\cdot), \\
\phi_2(\cdot) &= {\rm Upsample}({\rm Conv}(\cdot)), \\
\varphi(\cdot) \hspace{0.008\linewidth} &= {\rm ReLU}({\rm Conv}(\cdot)), \\
\forall i \hspace{0.02\linewidth} &\in \{1,2,3,4,5\}.
\end{aligned}
\end{equation}
Note that we have $\widetilde{\mathcal{S}}^6 = \mathcal{S}^6$,
because $\mathcal{S}^6$ is the last block in the encoder path 
and also the first block in the decoder path. 
In this way, the proposed encoder-decoder lets top contextual 
information flow into the lower layers, so the lower layers are 
expected to emphasize the details of salient objects in an image.
Here, both two sequential convolution layers ($\varphi(\cdot)$) at the  
decoder side  $\widetilde{\mathcal{S}}^i$ are with kernel size 
of $K_i \times K_i$ and output channels of $C_i$.
We will discuss the parameter settings in detail in the experiment part.

\begin{table}[!tb]
\centering
\renewcommand{\tabcolsep}{4mm}
\caption{Network configurations.} \label{tab:net_config}
\begin{tabular}{c|c|c|c} \Xhline{1.0pt}
Side & $C_i$ & $K_i \times K_i$ & Resolution \\ \hline
Side-output 1 & 64 & $3 \times 3$ & 1 \\ \hline
Side-output 2 & 128 & $3 \times 3$ & 1/2\\ \hline
Side-output 3 & 128 & $5 \times 5$ & 1/4 \\ \hline
Side-output 4 & 128 & $5 \times 5$ & 1/8 \\ \hline
Side-output 5 & 128 & $5 \times 5$ & 1/16 \\ \hline
Side output 6 & - & - & 1/32 \\ \Xhline{1.0pt}
\end{tabular}
\end{table}

\subsection{Deeply-supervised Nonlinear Aggregation} \label{sec:dna}
Instead of linearly aggregating side-output predictions at multiple 
sides as in previous literature 
\cite{zhang2018progressive,wang2018salient,liu2018picanet,islam2018revisiting,zhang2017amulet,liu2016dhsnet,he2017delving},
we propose to aggregate the side-output features in a nonlinear way.
The proposed DNA module is displayed in the dotted box 
of \figref{fig:framework}.
Specifically, we first adopt a $3 \times 3$ convolution for each 
$\widetilde{\mathcal{S}}^i$ to adjust the number of channels.
Then, the feature maps are upsampled into the same size of 
the original image to generate \textbf{side-output features}. 
The side-output  features can predict saliency maps using a 
simple $1 \times 1$ convolution. 
In the training phase, deep supervision is added for these 
predicted maps.

We concatenate all side-output features to construct hybrid features
that contain rich multi-scale and multi-level information.
One of the key ideas in our nonlinear aggregation is that we 
use asymmetric convolution that decomposes a standard 
two-dimensional convolution into two one-dimensional convolutions.
That is to say, a $n \times n$ convolution is decomposed into two 
sequential convolutions with kernel sizes of $1 \times n$ and 
$n \times 1$.
Here, the reasons why we use asymmetric convolution are twofold.
On one hand, in the experiments, we find large kernel size in the 
DNA module can improve performance, and we believe it is 
because hybrid feature maps have large resolution, \ie, the same 
resolution as the original image.
On the other hand, convolutions with large kernel sizes are very 
time-consuming for large feature maps.
According to the above analyses, we set $n=7$ for 
asymmetric convolutions rather than small kernel sizes. 
Larger kennel sizes than $n=7$ will only lead to little 
accuracy improvement while causing more computational load.
The effectiveness of this choice has been validated 
in \secref{sec:ablation} where we try different settings 
of $n$ and asymmetric/standard convolutions.
We use two groups of asymmetric convolutions, each of which 
consists of a $1 \times 7$ and a $7 \times 1$ convolution.
With a $300 \times 300$ input image, the number of FLOPs
(multiply-adds) for these asymmetric convolutions is 13.8G, 
while the number of FLOPs will be 60.4G if we use the standard
two-dimensional $7 \times 7$ convolutions.
At last, we connect a $1 \times 1$ convolution after the asymmetric 
convolutions to predict the final saliency maps.

\begin{table*}[!tb]
\centering
\renewcommand{\tabcolsep}{3.3mm}
\caption{Comparison between the proposed DNA and 16 competitors in terms 
of the metrics of $F_\beta$ and MAE on six datasets. 
We report results on both VGG16 \cite{simonyan2014very} backbone 
and ResNet-50 \cite{he2016deep} backbone.
The top three models in each column are highlighted in \first{red},
\second{green} and \third{blue}, respectively. 
For ResNet-50 based methods, we only highlight the top performance.
} \label{tab:eval_fb_mae}
\begin{tabular}{c||c|c|c|c|c|c|c|c|c|c|c|c} \Xhline{1.0pt}
	\multirow{2}*{Methods} & \multicolumn{2}{c|}{DUTS-TE} 
	& \multicolumn{2}{c|}{ECSSD} & \multicolumn{2}{c|}{HKU-IS} 
	& \multicolumn{2}{c|}{DUT-O} & \multicolumn{2}{c|}{SOD} 
	& \multicolumn{2}{c}{THUR15K} \\ \cline{2-13}
    & $F_\beta$ & MAE & $F_\beta$ & MAE & $F_\beta$ & MAE
    & $F_\beta$ & MAE & $F_\beta$ & MAE & $F_\beta$ & MAE
    \\ \hline
    \multicolumn{13}{c}{Non-deep learning} 
    \\ \hline
    DRFI \cite{jiang2013salient} & 0.649 
    & 0.154 & 0.777 & 0.161 & 0.774 & 0.146 & 0.652 & 0.138 & 0.704 
    & 0.217 & 0.670 & 0.150 
    \\ \hline
    \multicolumn{13}{c}{VGG16 \cite{simonyan2014very} backbone} 
    \\ \hline
    MDF \cite{li2015visual} & 0.707 & 0.114 & 0.807 
    & 0.138 & - & - & 0.680 & 0.115 & 0.764 & 0.182 & 0.669 & 0.128 
    \\
    LEGS \cite{wang2015deep} & 0.652 & 0.137 
    & 0.830 & 0.118 & 0.766 & 0.119 & 0.668 & 0.134 & 0.733 & 0.194 
    & 0.663 & 0.126 
    \\
    DCL \cite{li2016deep} & 0.785 & 0.082 & 0.895 & 0.080 
    & 0.892 & 0.063 & 0.733 & 0.095 & 0.831 & 0.131 & 0.747 & 0.096 
    \\
    DHS \cite{liu2016dhsnet} & 0.807 & 0.066 & 0.903 
    & 0.062 & 0.889 & 0.053 & - & -  & 0.822 & 0.128 & 0.752 & 0.082 
    \\
    ELD \cite{lee2016deep} & 0.727 & 0.092 & 0.866 & 0.081 
    & 0.837 & 0.074 & 0.700 & 0.092 & 0.758 & 0.154 & 0.726 & 0.095 
    \\
    RFCN \cite{wang2016saliency} & 0.782 & 0.089 & 0.896 
    & 0.097 & 0.892 & 0.080 & 0.738 & 0.095 & 0.802 & 0.161 & 0.754 & 0.100 
    \\
    NLDF \cite{luo2017non} & 0.806 & 0.065 & 0.902 
    & 0.066 & 0.902 & 0.048 & 0.753 & 0.080 & 0.837 & 0.123 & 0.762 & 0.080 
    \\
    DSS \cite{hou2019deeply} & 0.827 & \third{0.056} & 0.915 
    & \third{0.056} & \third{0.913} & \second{0.041} & \third{0.774} 
    & \third{0.066} & \third{0.842} & \third{0.122} & 0.770 & \second{0.074} 
    \\
    Amulet \cite{zhang2017amulet} & 0.778 & 0.085 & 0.913 
    & 0.061 & 0.897 & 0.051 & 0.743 & 0.098 & 0.795 & 0.144 & 0.755 & 0.094 
    \\
    UCF \cite{zhang2017learning} & 0.772 & 0.112 & 0.901 
    & 0.071 & 0.888 & 0.062 & 0.730 & 0.120 & 0.805 & 0.148 & 0.758 & 0.112 
    \\
    PiCA \cite{liu2018picanet} & \second{0.837} & \second{0.054} 
    & \second{0.923} & \second{0.049} & \second{0.916} & \third{0.042} 
    & 0.766 & 0.068 & 0.836 & \first{0.102} & \second{0.783} & 0.083 
    \\
    C2S \cite{li2018contour} & 0.811 & 0.062 & 0.907 
    & 0.057 & 0.898 & 0.046 & 0.759 & 0.072 & 0.819 & \third{0.122} 
    & \third{0.775} & 0.083 
    \\
    RAS \cite{chen2018reverse}  
    & \third{0.831} & 0.059 & \third{0.916} & 0.058 & \third{0.913} 
    & 0.045 & \second{0.785} & \second{0.063} & \second{0.847} & 0.123 
    & 0.772 & \third{0.075} 
    \\
    \textbf{DNA} & \first{0.865} 
    & \first{0.044} & \first{0.935} & \first{0.041} & \first{0.930} 
    & \first{0.031} & \first{0.799} & \first{0.056} & \first{0.853} 
    & \second{0.107} & \first{0.793} & \first{0.069}
    \\ \hline
    \multicolumn{13}{c}{ResNet-50 \cite{he2016deep} backbone} 
    \\ \hline
    SRM \cite{wang2017stagewise} & 0.826 & 0.059 & 0.914 
    & 0.056 & 0.906 & 0.046 & 0.769 & 0.069 & 0.840 & 0.126 & 0.778 & 0.077 
    \\
    BRN \cite{wang2018detect} & 0.827 & 0.050 & 0.919 
    & 0.043 & 0.910 & 0.036 & 0.774 & 0.062 & 0.843 & \first{0.103} 
    & 0.769 & 0.076 
    \\ 
    PiCA \cite{liu2018picanet} & 0.853 & 0.050 & 0.929 
    & 0.049 & 0.917 & 0.043 & 0.789 & 0.065 & 0.852 & \first{0.103} 
    & 0.788 & 0.081 
    \\
    \textbf{DNA} & \first{0.873} 
    & \first{0.040} & \first{0.938} & \first{0.040} & \first{0.934} 
    & \first{0.029} & \first{0.805} & \first{0.056} & \first{0.855} 
    & 0.110 & \first{0.796} & \first{0.068}
    \\ \Xhline{1.0pt}
\end{tabular}
\end{table*}

In training, we adopt class-balanced cross-entropy loss 
\cite{xie2017holistically} to supervise all side-output 
and final fused predictions.
Since convolutions in the DNA module are followed by nonlinear 
activation (\ie, ReLU), the aggregation of 
side-output features is nonlinear.
Although there are several nonlinear functions that can be used, 
such as ReLU, PReLU, and LeakyReLU, in this paper, 
we simply use the most common ReLU function to demonstrate 
the necessity of nonlinear side-output aggregation.
The traditional linear side-output prediction aggregation can only 
linearly combine multi-scale predictions, while the proposed 
nonlinear side-output feature aggregation can make use of the 
complementary multi-scale features for final prediction and  
is thus more effective.
With the simple encoder-decoder in \secref{sec:basenet},
DNA performs favorably against previous methods. 
Note that previous methods 
\cite{zhang2018progressive,wang2018salient,liu2018picanet,islam2018revisiting}
usually present various network architectures, modules, 
and operations to improve performance, but in this paper, 
DNA only applies a simply-modified U-Net as base network.

\section{Experiments}
\subsection{Experimental Setup} \label{sec:experiment_setup}
\noindent\textbf{Implementation details.}
The detailed configurations for $K_i$ and $C_i$ can be found 
in \tabref{tab:net_config}.
The large kernel size at top sides is helpful to accuracy.
When $i=1,2$, $K_i \times K_i$ equals to $3 \times 3$;
When $i=3,4,5$, $K_i \times K_i$ equals to $5 \times 5$.
The $C_i$ values for $i=1,\cdots,5$ are 64, 128, 128, 128 and 128, 
respectively.
Since side-output prediction results have not been used, 
we remove these side-output prediction units in the test phase.
However, we remain them in the training phase, because 
deep supervision can help the training and improve the accuracy 
of the final saliency prediction, as demonstrated 
in \secref{sec:ablation}.

We implement our network using the well-known 
Caffe \cite{jia2014caffe} framework.
The convolution layers in the original VGG16 
\cite{simonyan2014very} are initialized using 
the pretrained ImageNet model \cite{deng2009imagenet}.
The weights of other layers are initialized from the zero-mean 
Gaussian distribution with standard deviation 0.01. 
Biases are initialized to 0.
The upsampling operations are implemented by deconvolution 
layers with frozen bilinear interpolation kernels.
Since the deconvolution layers do not need training, we 
exclude them when computing the number of parameters.
The network is optimized using SGD with learning rate policy
of \textit{poly}, in which the current learning rate equals 
the base one multiplying $(1-curr\_iter/max\_iter)^{power}$.
The hyper parameters $power$ and $max\_iter$ are set to 0.9 
and 20000, respectively, so that the training takes 20000 
iterations in total.
The initial learning rate is set to 1e-7 that is the 
maximum value to keep the network from training exploding 
(\ie, larger values will cause the well-known ``Nan'' error).
We follow previous saliency detection methods 
\cite{hou2019deeply,zhang2018progressive,liu2018picanet,chen2018reverse,liu2016dhsnet,lee2016deep,wang2018detect,zhang2017supervision,zhang2018deep,wang2015deep,chen2016disc,liu2018deep,chen2017look,li2016deep,su2019selectivity,liu2019simple}
to set the momentum and weight decay to the typical values 
of 0.9 and 0.0005 \cite{krizhevsky2012imagenet,simonyan2014very},
respectively.
All experiments are performed on a TITAN Xp GPU.

\myPara{Datasets.}
We extensively evaluate our method on six popular datasets,
including DUTS \cite{wang2017learning}, 
ECSSD \cite{yan2013hierarchical}, 
SOD \cite{movahedi2010design}, 
HKU-IS \cite{li2015visual}, 
THUR15K \cite{cheng2014salientshape} 
and DUT-O (or DUT-OMRON) \cite{yang2013saliency}.
These six datasets consist of 15572, 1000, 300, 4447, 6232 
and 5168 natural complex images with corresponding pixel-wise 
ground truth labeling.
Among them, the DUTS dataset \cite{wang2017learning} is a very 
recent dataset consisting of 10553 training images 
and 5019 test images in very complex scenarios.
For a fair comparison, we follow recent studies  
\cite{wang2018detect,liu2018picanet,wang2017stagewise,zeng2018learning}
to use DUTS training set for training and test on the DUTS test set 
(DUTS-TE) and other datasets.

\begin{table*}[!ht]
\centering
\renewcommand{\tabcolsep}{4.9mm}
\caption{Comparison between the proposed DNA and 16 competitors in terms 
of $F_\beta^\omega$-measure \cite{margolin2014evaluate} on six datasets. 
The unit of the number of parameters (\#Param) is million (M), and the unit 
of speed is frame per second (fps). 
We report results on both VGG16 \cite{simonyan2014very} backbone 
and ResNet-50 \cite{he2016deep} backbone.
The top three models in each column are highlighted in \first{red},
\second{green} and \third{blue}, respectively. 
For ResNet-50 based methods, we only highlight the top performance.
} \label{tab:eval_Fbw}
\begin{tabular}{c||c|c||c|c|c|c|c|c} \Xhline{1.0pt}
	Methods & \multicolumn{1}{c|}{\#Param} & \multicolumn{1}{c||}{Speed} 
	& DUTS-TE & ECSSD & HKU-IS & DUT-O & SOD & THUR15K \\ \hline
    \multicolumn{9}{c}{Non-deep learning} 
    \\ \hline
    DRFI \cite{jiang2013salient} & \multicolumn{1}{c|}{-} & 1/8 
    & 0.378 & 0.548 & 0.504 & 0.424 & 0.450 & 0.444 \\ \hline
    \multicolumn{9}{c}{VGG16 \cite{simonyan2014very} backbone} 
    \\ \hline
    MDF \cite{li2015visual} & 56.86 & 1/19 & 0.507 & 0.619 & - 
    & 0.494 & 0.528 & 0.508 \\
    LEGS \cite{wang2015deep} & \first{18.40} & 0.6 & 0.510 & 0.692 
    & 0.616 & 0.523 & 0.550 & 0.538 \\
    DCL \cite{li2016deep} & 66.24 & 1.4 & 0.632 & 0.782 & 0.770 
    & 0.584 & 0.669 & 0.624 \\
    DHS \cite{liu2016dhsnet} & 94.04 & 10.0 & 0.705 & 0.837 & 0.816 
    & - & 0.685 & 0.666 \\
    ELD \cite{lee2016deep} & 43.09 & 1.0 & 0.607 & 0.783 & 0.743 
    & 0.593 & 0.634 & 0.621 \\
    RFCN \cite{wang2016saliency} & 134.69 & 0.4 & 0.586 & 0.725 
    & 0.707 & 0.562 & 0.591 & 0.592 \\
    NLDF \cite{luo2017non} & 35.49 & \third{18.5} & 0.710 & 0.835 
    & 0.838 & 0.634 & 0.708 & 0.676 \\
    DSS \cite{hou2019deeply} & 62.23 & 7.0 & 0.700 & 0.832 & 0.821 
    & 0.643 & 0.698 & 0.662 \\
    Amulet \cite{zhang2017amulet} & 33.15 & 9.7 & 0.657 & 0.839 
    & 0.817 & 0.626 & 0.674 & 0.650 \\
    UCF \cite{zhang2017learning} & 23.98 & 12.0 & 0.595 & 0.805 
    & 0.779 & 0.574 & 0.673 & 0.613 \\
    PiCA \cite{liu2018picanet} & 32.85 & 5.6 & \second{0.745} 
    & \second{0.862} & \third{0.847} & \third{0.691} 
    & \second{0.721} & \third{0.688} \\
    C2S \cite{li2018contour} & 137.03 & 16.7 & 0.717 & 0.849 
    & 0.835 & 0.663 & 0.700 & 0.685 \\
    RAS \cite{chen2018reverse} & \third{20.13} & \second{20.4} 
    & \third{0.739} & \third{0.855} & \second{0.850} & \second{0.695} 
    & \third{0.718} & \second{0.691} \\ 
    \textbf{DNA} & \second{20.06} & \first{25.0} & \first{0.797} 
    & \first{0.897} & \first{0.889} & \first{0.729} & \first{0.755} 
    & \first{0.723} \\ \hline
    \multicolumn{9}{c}{ResNet-50 \cite{he2016deep} backbone} 
    \\ \hline
    SRM \cite{wang2017stagewise} & 43.74 & 12.3 & 0.721 & 0.849 
    & 0.835 & 0.658 & 0.670 & 0.684 \\
    BRN \cite{wang2018detect} & 126.35 & 3.6 & 0.774 & 0.887 & 0.876 
    & 0.709 & 0.738 & 0.712 \\ 
    PiCA \cite{liu2018picanet} & 37.02 & 4.4 & 0.754 & 0.863 & 0.841 
    & 0.695 & 0.722 & 0.690 \\ 
    \textbf{DNA} & \first{29.31} & \first{12.8} & \first{0.810} 
    & \first{0.901} & \first{0.898} & \first{0.735} & \first{0.755} 
    & \first{0.730} \\ \Xhline{1.0pt}
\end{tabular}
\end{table*}

\myPara{Evaluation criteria.}
We utilize three evaluation metrics to evaluate our method 
as well as other recent salient object detectors, 
including max F-measure score ($F_\beta$), mean absolute 
error (MAE), and the weighted $F_\beta^\omega$-measure 
score \cite{margolin2014evaluate}.

Given a predicted saliency map with continuous probability values, 
we can convert it into binary maps with arbitrary thresholds and 
computing corresponding precision/recall values.
Taking the average of precision/recall values over all images 
in a dataset, we can get many mean 
precision/recall pairs.
F-measure is an overall performance indicator: 
\begin{equation}
F_\beta = \frac{(1+\beta^2) \times {\rm Precision} \times {\rm Recall}}{
\beta^2 \times {\rm Precision} + {\rm Recall}},
\end{equation}
in which $\beta^2$ is usually set to 0.3 to emphasize more on precision.
We follow recent studies 
\cite{luo2017non,hou2019deeply,zhang2017amulet,zhang2017learning,liu2018picanet,li2018contour,chen2018reverse} 
to report max $F_\beta$ across different thresholds.

Given a saliency map $S$ and the corresponding ground truth $G$ 
that are normalized to [0, 1], MAE can be calculated as 
\begin{equation}
{\rm MAE} = \frac{1}{H \times W} \sum_{i=1}^{H} \sum_{j=1}^{W} |S(i,j)-G(i,j)|,
\end{equation}
where $H$ and $W$ are height and width, respectively.
$S(i,j)$ denotes the saliency score at location $(i,j)$,
similar to $G(i,j)$.

As demonstrated in \cite{margolin2014evaluate}, traditional 
evaluation metrics easily suffer from the interpolation flaw, 
dependency flaw, and equal-importance flaw.
Hence the weighted $F_\beta^\omega$-measure score is proposed 
to amend these flaws.
We follow \cite{gong2015saliency,hu2017deep,tu2016real,liu2018picanet} 
to adopt $F_\beta^\omega$-measure as a metric with the default 
settings.

\subsection{Performance Comparison}
We compare our proposed salient object detector with 16 recent 
competitive saliency models, including 
DRFI \cite{jiang2013salient}, MDF \cite{li2015visual}, 
LEGS \cite{wang2015deep}, DCL \cite{li2016deep}, 
DHS \cite{liu2016dhsnet}, ELD \cite{lee2016deep}, 
RFCN \cite{wang2016saliency}, NLDF \cite{luo2017non}, 
DSS \cite{hou2019deeply}, SRM \cite{wang2017stagewise}, 
Amulet \cite{zhang2017amulet}, UCF \cite{zhang2017learning}, 
BRN \cite{wang2018detect}, PiCA \cite{liu2018picanet}, 
C2S \cite{li2018contour} and RAS \cite{chen2018reverse}. 
Among them, DRFI \cite{jiang2013salient} is the best-known
non-deep-learning based method, and the other 15 models are 
all based on deep learning.
We do not report MDF \cite{li2015visual} results on the 
HKU-IS \cite{li2015visual} dataset because MDF uses a part of HKU-IS
for training.
Due to the same reason, we do not report DHS \cite{liu2016dhsnet} 
results on the DUT-O \cite{yang2013saliency} dataset.
Since SRM \cite{wang2017stagewise} and BRN \cite{wang2018detect}
are built based on the ResNet-50 \cite{he2016deep} backbone, 
we also report the results of the ResNet-50 version of the proposed 
DNA and PiCA \cite{liu2018picanet} for a fair comparison.
All previous methods are tested using their publicly available code 
and the pretrained models released by the authors with default settings.

\begin{figure*}[!tb]
    \centering
    \leftline{\scriptsize Simple Scenes $\mid$ Center Bias}
    \includegraphics[width=\linewidth]{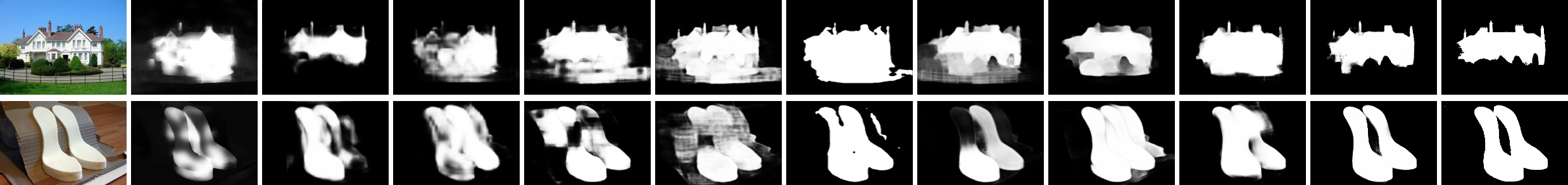}
    \\ \vspace{-0.12in} \rule{\linewidth}{0.2mm} 
    \leftline{\scriptsize Thin Objects $\mid$ Thin Object Parts 
    $\mid$ Large Objects} \\ \vspace{0.04in}
    \includegraphics[width=\linewidth]{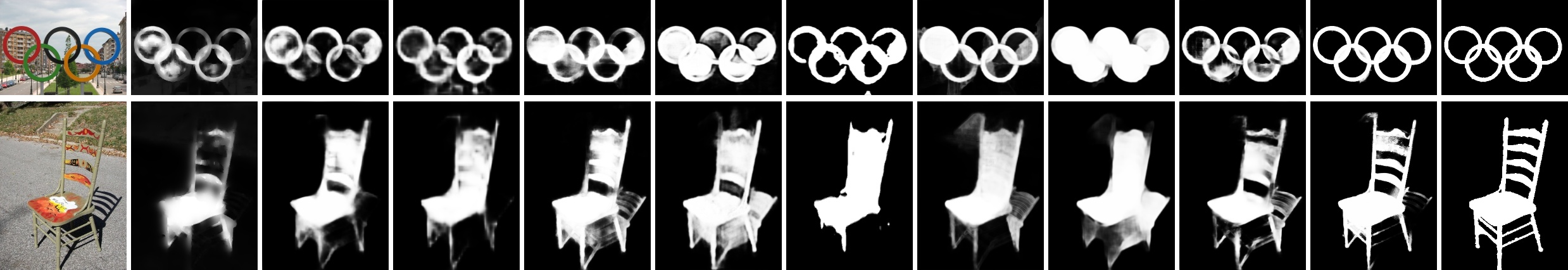} 
    \\ \vspace{-0.12in} \rule{\linewidth}{0.2mm} 
    \leftline{\scriptsize Low Contrast $\mid$ Complex Scenes
    $\mid$ Complex Textures} \\ \vspace{0.04in}
    \includegraphics[width=\linewidth]{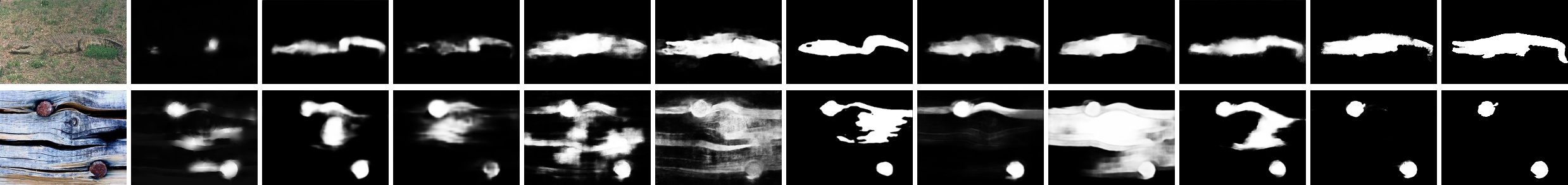}
    \\ \vspace{-0.12in} \rule{\linewidth}{0.2mm} 
    \leftline{\scriptsize Large Objects $\mid$ Confusing Background} 
    \\ \vspace{0.04in}
    \includegraphics[width=\linewidth]{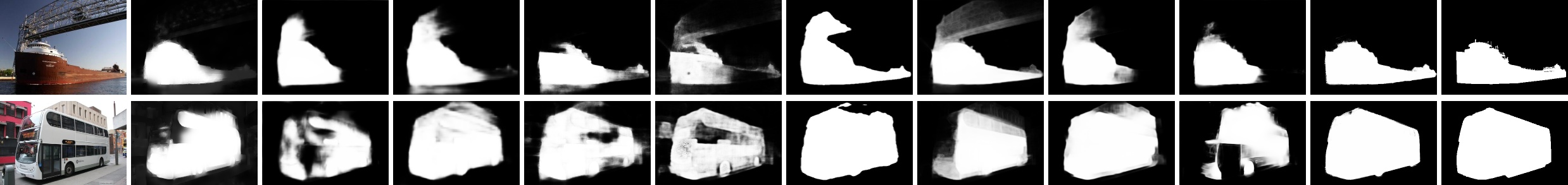}
    \\ \vspace{-0.12in} \rule{\linewidth}{0.2mm} 
    \leftline{\scriptsize Multiple Objects $\mid$ Complex Scenes} 
    \\ \vspace{0.04in}
    \includegraphics[width=\linewidth]{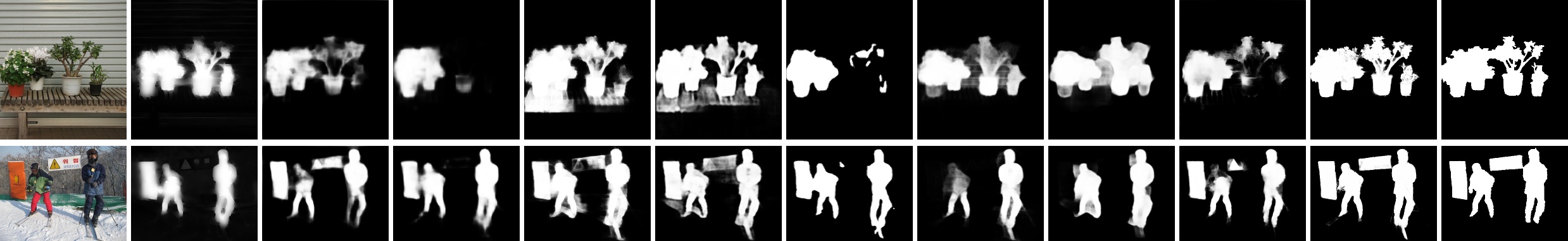}
    \\ \vspace{-0.12in} \rule{\linewidth}{0.2mm} 
    \leftline{\scriptsize Complex Scenes $\mid$ Complex Textures 
    $\mid$ Multiple Objects} \\ \vspace{0.04in}
    \includegraphics[width=\linewidth]{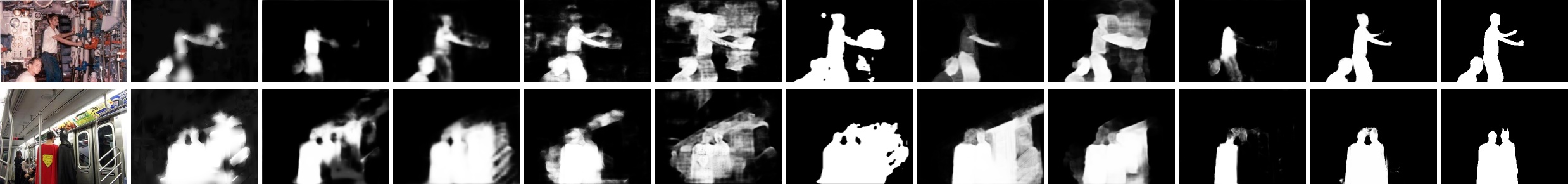}
    \\ \vspace{-0.12in} \rule{\linewidth}{0.2mm} 
    \leftline{\scriptsize Confusing Background $\mid$ Low Contrast} 
    \\ \vspace{0.04in}
    \includegraphics[width=\linewidth]{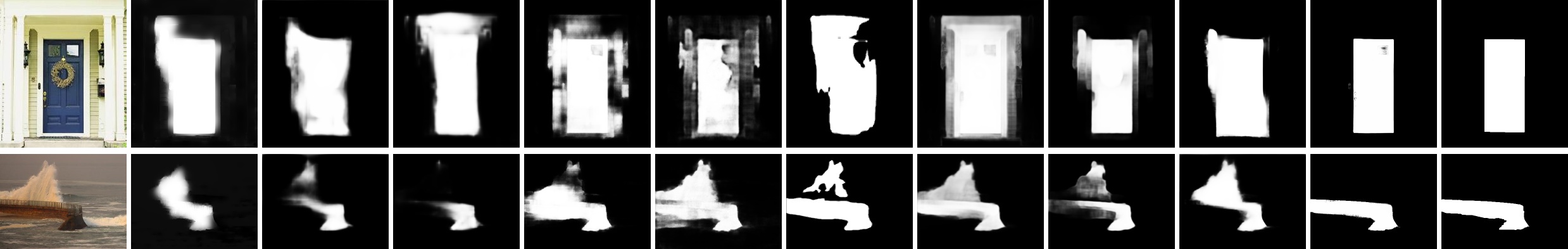}
    \\ \vspace{-0.12in} \rule{\linewidth}{0.2mm} 
    \leftline{\scriptsize Abnormal Brightness $\mid$ Large Objects} 
    \\ \vspace{0.04in}
    \includegraphics[width=\linewidth]{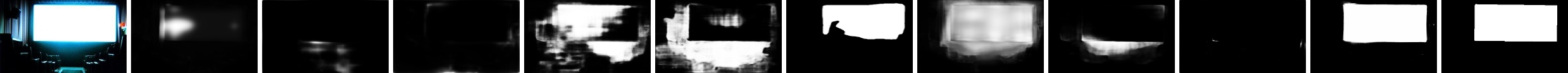}
    \\ \vspace{-0.05in}
    \leftline{\scriptsize\hspace{0.12in} Image \hspace{0.29in} RFCN
    \hspace{0.31in} DSS \hspace{0.34in} SRM \hspace{0.29in} Amulet
    \hspace{0.29in} UCF \hspace{0.34in} BRN \hspace{0.33in} PiCA 
    \hspace{0.34in} C2S \hspace{0.34in} RAS \hspace{0.34in} 
    \textbf{Ours} \hspace{0.36in} GT}
    \caption{Qualitative comparison between DNA and recent 
    competitive methods.
    Here, GT represents ground-truth saliency maps.
    }
    \label{fig:samples}
\end{figure*}

\begin{table*}[!t]
\centering
\renewcommand{\tabcolsep}{2.7mm}
\caption{Ablation studies. U-Net means the standard U-Net 
\cite{ronneberger2015u} with VGG16 backbone. 
If removing the DNA module and deep supervision, the proposed 
network in \figref{fig:framework} becomes an encoder-decoder 
network that is called \textit{Encoder-Decoder}. 
\textit{Encoder-Decoder w/ K3} replaces all the convolutions 
at the top sides of Encoder-Decoder with $3 \times 3$ convolutions. 
\textit{Encoder-Decoder w/ lin} means we replace the DNA module 
in \figref{fig:framework} with traditional linear aggregation 
in \cite{xie2017holistically}.} \label{tab:improvement}
\begin{tabular}{c||c|c|c|c|c|c|c|c|c|c|c|c} \Xhline{1.0pt}
    \multirow{2}*{Methods} & \multicolumn{2}{c|}{DUTS-TE} 
	& \multicolumn{2}{c|}{ECSSD} & \multicolumn{2}{c|}{HKU-IS} 
	& \multicolumn{2}{c|}{DUT-O} & \multicolumn{2}{c|}{SOD} 
	& \multicolumn{2}{c}{THUR15K} \\ \cline{2-13}
    & $F_\beta$ & MAE & $F_\beta$ & MAE & $F_\beta$ & MAE
    & $F_\beta$ & MAE & $F_\beta$ & MAE & $F_\beta$ & MAE \\ \hline
    U-Net & 0.793 & 0.080 & 0.890 & 0.065 & 0.894 & 0.051 
    & 0.723 & 0.101 & 0.811 & 0.115 & 0.758 & 0.099 \\ \hline
    Encoder-Decoder w/ K3 & 0.766 & 0.101 & 0.869 & 0.081 & 0.876 
    & 0.064 & 0.687 & 0.129 & 0.778 & 0.131 & 0.736 & 0.112 \\ \hline
    Encoder-Decoder & 0.831 & 0.053 & 0.911 & 0.052 & 0.916 
    & 0.037 & 0.754 & 0.073 & 0.830 & 0.117 & 0.780 & 0.077 \\ \hline
    Encoder-Decoder w/ lin & 0.844 & 0.048 & 0.921 & 0.050 & 0.917 
    & 0.034 & 0.765 & 0.066 & 0.839 & 0.120 & 0.785 & 0.071 \\ \hline
    DNA w/o Deep Supervision & 0.867 & 0.042 & 0.932 & 0.041
    & 0.927 & 0.032 & 0.788 & 0.059 & 0.860 & 0.103
    & 0.794 & 0.068
    \\ \hline
    DNA & 0.865 & 0.044 & 0.935 & 0.041 & 0.930 & 0.031 
    & 0.799 & 0.056 & 0.853 & 0.107 & 0.793 & 0.069
    \\ \Xhline{1.0pt}
\end{tabular}
\end{table*}

\begin{table*}[!htb]
\centering
\caption{Ablation studies for various parameter settings. 
The unit of the number of parameters (\#Param) is million (M), and the unit 
of speed is frame per second (fps).} \label{tab:parameter_settings}
\resizebox{\textwidth}{!}{
\begin{tabular}{c||P{3.2}|P{2.1}||c|c|c|c|c|c|c|c|c|c|c|c} \Xhline{1.0pt}
	\multirow{2}*{Methods} & \multicolumn{1}{c|}{\multirow{2}*{\#Param}}
    & \multicolumn{1}{c||}{\multirow{2}*{Speed}} 
    & \multicolumn{2}{c|}{DUTS-TE} & \multicolumn{2}{c|}{ECSSD} 
    & \multicolumn{2}{c|}{HKU-IS} & \multicolumn{2}{c|}{DUT-O} 
    & \multicolumn{2}{c|}{SOD} & \multicolumn{2}{c}{THUR15K} 
    \\ \cline{4-15}
    &&& $F_\beta$ & MAE & $F_\beta$ & MAE & $F_\beta$ & MAE
    & $F_\beta$ & MAE & $F_\beta$ & MAE & $F_\beta$ & MAE
    \\ \hline
    \#1 & 18.49 & 27.0 & 0.859 & 0.044 & 0.932 & 0.041 & 0.928 & 0.031 
    & 0.796 & 0.057 & 0.855 & 0.105 & 0.790 & 0.069
    \\ \rowcolor{mycolor}
    \#2 & 20.06 & 25.0 & 0.865 & 0.044 & 0.935 & 0.041 & 0.930 
    & 0.031 & 0.799 & 0.056 & 0.853 & 0.107 & 0.793 & 0.069
    \\
    \#3 & 27.88 & 22.7 & 0.866 & 0.043 & 0.936 & 0.041 & 0.930 
    & 0.031 & 0.799 & 0.056 & 0.861 & 0.106 & 0.792 & 0.069
    \\
    \#4 & 41.41 & 18.2 & 0.864 & 0.044 & 0.935 & 0.041 & 0.931 
    & 0.030 & 0.800 & 0.056 & 0.857 & 0.105 & 0.792 & 0.069 
    \\ \Xhline{1.0pt}
\end{tabular}}
\end{table*}

\begin{table}[!tb]
\centering
\renewcommand{\tabcolsep}{1.8mm}								
\caption{Parameter settings for ablation studies in 
\tabref{tab:parameter_settings}. The default setting in this 
paper is highlighted in dark.} \label{tab:ablation_config}
\begin{tabular}{c||c|>{\columncolor{mycolor}}c|c|c} \Xhline{1.0pt}
No. & \#1 & \#2 & \#3 & \#4 \\ \hline 
Side 1 & $(3 \times 3, 64)$ & $(3 \times 3, 64)$ 
& $(3 \times 3, 64)$ & $(3 \times 3, 64)$ \\
Side 2 & $(3 \times 3, 128)$ & $(3 \times 3, 128)$ 
& $(3 \times 3, 128)$ & $(3 \times 3, 128)$ \\
Side 3 & $(3 \times 3, 128)$ & $(5 \times 5, 128)$ 
& $(5 \times 5, 128)$ & $(5 \times 5, 256)$ \\
Side 4 & $(3 \times 3, 128)$ & $(5 \times 5, 128)$ 
& $(5 \times 5, 256)$ & $(5 \times 5, 256)$ \\
Side 5 & $(3 \times 3, 128)$ & $(5 \times 5, 128)$ 
& $(5 \times 5, 256)$ & $(5 \times 5, 512)$ \\
Side 6 & $(192, 128)$ & $(192, 128)$ & $(256, 256)$ 
& $(256, 256)$ \\ \Xhline{1.0pt}
\end{tabular}
\end{table}

\myPara{F-measure and MAE.}
\tabref{tab:eval_fb_mae} summarizes the numeric comparison 
in terms of F-measure ($F_\beta$) and MAE on six datasets.
DNA can significantly outperform other competitors in most cases, 
which demonstrates its effectiveness. 
With VGG16 \cite{simonyan2014very} backbone, the $F_\beta$ 
values of DNA are 2.8\%, 1.2\%, 1.4\%, 1.4\%, 0.6\% and 1.0\% 
higher than the second best method on the DUTS-TE, ECSSD, HKU-IS, 
DUT-O, SOD and THUR15K datasets, respectively.
As can be seen, DNA also achieves the best performance in terms 
of MAE metric except on the SOD dataset where DNA performs 
slightly worse than PiCA \cite{liu2018picanet}.
Overall, PiCA \cite{liu2018picanet} seems to achieves the second place. 
With the ResNet-50 backbone, DNA still performs better than previous 
competitors, indicating DNA is robust to different network architectures.
Therefore, we suggest the future salient object detectors 
using nonlinear side-output aggregation instead of the traditional 
linear aggregation.

\myPara{Weighted $F_\beta^\omega$-measure.}
The weighted $F_\beta^\omega$-measure is also a commonly-used 
saliency evaluation metric.
In \tabref{tab:eval_Fbw}, we evaluate DNA and above-mentioned 
competitors using the $F_\beta^\omega$-measure.
The VGG16 version of DNA achieves 5.2\%, 3.5\%, 3.9\%, 3.4\%, 
3.4\% and 3.2\% higher $F_\beta^\omega$-measure than the second 
best performance on the DUTS-TE, ECSSD, HKU-IS, DUT-O, SOD 
and THUR15K datasets, respectively. 
For ResNet-50 version, DNA achieves 3.6\%, 1.4\%, 2.2\%, 2.6\%, 
1.7\% and 1.8\% better $F_\beta^\omega$-measure than previous 
competitors. 
Note that the network of DNA is very simple, making 
it easy to be followed and applied to other vision tasks.

\myPara{Number of parameters and runtime.}
As shown in \tabref{tab:eval_Fbw}, DNA has fewer parameters, 
\ie, about 20M parameters with VGG16 backbone and 29M parameters 
with ResNet-50 backbone. 
DNA also runs faster than other methods, achieving 25fps with 
VGG16 and 12.8fps with ResNet-50.

\myPara{Qualitative comparison.}
To visually exhibit the superiority of the proposed DNA over 
previous methods, we select some representative images from 
various datasets to incorporate a variety of difficult 
circumstances, including complicated scenes, salient objects with 
thin structures, low contrast between foreground and background,
multiple objects with different sizes, scenes with abnormal 
brightness, and \etc.
We display a qualitative comparison in \figref{fig:samples} 
where we split the selected images into multiple groups, 
each of which is with several tags to describe its properties. 
Taking all circumstances into account, the proposed DNA 
can segment the right salient objects with coherent boundaries
and connected regions, even in the complex, low-contrast, 
and abnormal scenes.
This is the reason why DNA behaves better than other methods 
in the above quantitative comparison.

\subsection{Ablation Studies} \label{sec:ablation}
\noindent\textbf{Nonlinear aggregation \vs linear aggregation.}
To demonstrate the effectiveness of nonlinear aggregation, 
we replace the DNA module in our network with the traditional 
linear side-output prediction aggregation \cite{xie2017holistically} 
to obtain a deeply-supervised encoder-decoder, 
\ie, \textit{Encoder-Decoder w/ lin}.
The results are shown in \tabref{tab:improvement}.
We can clearly see that nonlinear aggregation performs significantly 
better than linear aggregation, in terms of both $F_\beta$ and MAE.
A qualitative comparison between the linear side-output 
aggregation and nonlinear aggregation is shown in the $4^{\rm th}$ and 
$5^{\rm th}$ columns of \figref{fig:ablation_samples}.
The superiority of nonlinear aggregation can be clearly observed 
in various complicated scenarios.

\myPara{The proposed encoder-decoder \vs standard U-Net.}
If removing the DNA module and deep supervision, the proposed 
encoder-decoder is a simply modified version of 
U-Net \cite{ronneberger2015u}. 
First, we change the kernel size of all convolutions at top sides, 
\ie, $K_3 \times K_3$, $K_4 \times K_4$ and $K_5 \times K_5$,
into $3 \times 3$.
As displayed in \tabref{tab:improvement}, the resulting model, 
\textit{Encoder-Decoder w/ K3}, perform worse than the standard 
U-Net \cite{ronneberger2015u}. 
This could be because the proposed encoder-decoder has less feature 
channels and thus less parameters (U-Net  has 31.06M parameters). 
Next, we use the default kernel size of $5 \times 5$ for top sides.
The resulting model, \textit{Encoder-Decoder}, performs better than U-Net. 
This demonstrates large kernel size at the top sides is important 
for better performance.
We provide the qualitative comparison between 
\textit{Encoder-Decoder} and U-Net in the $2^{\rm nd}$ and $3^{\rm rd}$ 
columns of \figref{fig:ablation_samples}.
The proposed \textit{Encoder-Decoder} can predict better saliency maps.

\myPara{Encoder-decoder with or without deep supervision.}
In \tabref{tab:improvement}, \textit{Encoder-Decoder w/ lin}
performs better than \textit{Encoder-Decoder}, which can also
be seen in the $3^{\rm rd}$ and $4^{\rm th}$ columns of 
\figref{fig:ablation_samples}.
If removing the deep supervision in DNA, the resulting model
(\textit{DNA w/o Deep Supervision}) performs worse than DNA 
in most scenarios.
Therefore, deep supervision can consistently improve the 
saliency prediction performance.

\begin{figure}[!tb]
    \centering
    \includegraphics[width=\linewidth]{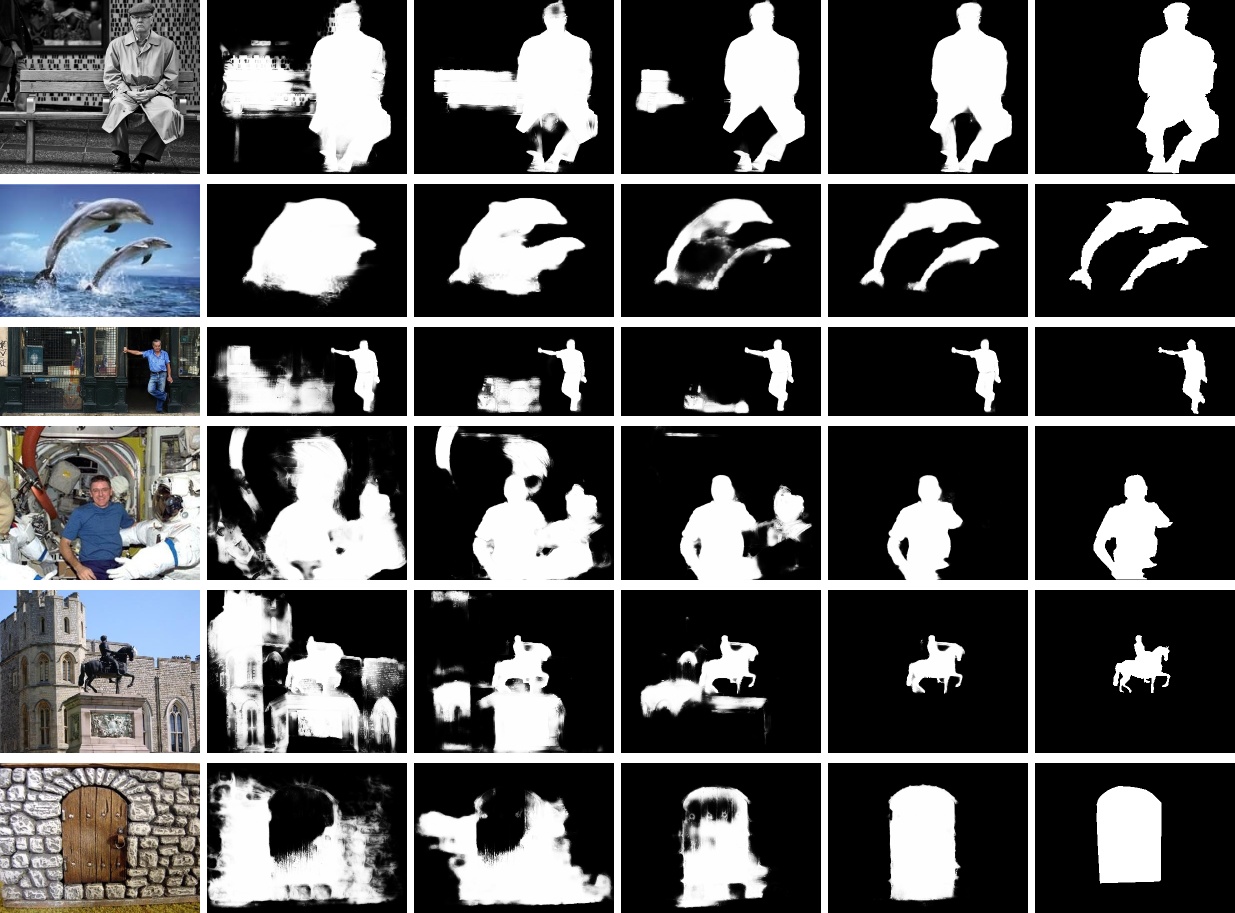} 
    \\ \vspace{-0.05in}
    \leftline{\scriptsize\hspace{0.13in} Image \hspace{0.27in} U-Net
    \hspace{0.33in} ED \hspace{0.26in} ED w/ lin \hspace{0.22in} DNA
    \hspace{0.34in} GT}
    \caption{Qualitative comparison between different model variants.
    ED: Encoder-Decoder; ED w/ lin: Encoder-Decoder w/ lin.
    From this figure, we can clearly see that the quality of saliency 
    prediction gradually increases from left to right.
    Since \textit{ED w/ lin} just replaces the nonlinear side-output 
    aggregation in DNA with linear aggregation, this figure demonstrates 
    the superiority of nonlinear aggregation in saliency detection.
    } \label{fig:ablation_samples}
\end{figure}

\begin{table}[!t]
\centering
\caption{Various convolution kernel sizes in the DNA module.}
\label{tab:aggregation}
\resizebox{\linewidth}{!}{
\begin{tabular}{c|c||c|c|c|c|c} \Xhline{1.0pt}
    Datasets & Metrics & $3 \times 3$ & $5 \times 5$ 
    & $7 \times 7$ & \makecell[c]{$1 \times 7$ \\ $7 \times 1$} 
    & \makecell[c]{$1 \times 9$ \\ $9 \times 1$} 
    \\ \hline
    \multirow{2}*{DUTS-TE} & $F_\beta$ & 0.861 & 0.863 & 0.865 
    & 0.865 & 0.864 
    \\ \cline{2-7}
    & MAE & 0.045 & 0.045 & 0.043 & 0.044 & 0.044
    \\ \hline
    \multirow{2}*{ECSSD} & $F_\beta$ & 0.930 & 0.933 & 0.935 
    & 0.935 & 0.935
    \\ \cline{2-7}
    & MAE & 0.042 & 0.041 & 0.041 & 0.041 & 0.040
    \\ \hline
    \multirow{2}*{DUT-O} & $F_\beta$ & 0.795 & 0.797 & 0.799 
    & 0.799 & 0.798 
    \\ \cline{2-7}
    & MAE & 0.058 & 0.058 & 0.057 & 0.056 & 0.056
    \\ \cline{1-7}
    \multicolumn{2}{c||}{Speed (fps)} & 27.8 & 23.2 & 19.6 
    & 25.0 & 20.4
    \\ \Xhline{1.0pt}
\end{tabular}}
\end{table}

\myPara{Parameter settings.}
To evaluate the effect of different parameter settings, we try 
various parameter settings in \tabref{tab:ablation_config}.
For side 1-5, we report the settings of $(K_i \times K_i, C_i)$.
For side 6, we report the settings of $C_6^{(1)},C_6^{(2)}$.
The evaluation results are summarized in \tabref{tab:parameter_settings}.
From the first and second experiment, we can see that large kernel 
sizes at top sides lead to better results, but the improvement 
is not as significant as in \tabref{tab:improvement} where 
deep supervision is not used.
From the third and fourth experiments, we find that introducing more 
parameters by increasing the convolution channels can generate 
slightly better results.
Considering the trade-off between the performance, the number 
of parameters and speed, we choose the second setting 
as our default parameters.

\myPara{The asymmetric convolutions in DNA module.}
In \tabref{tab:aggregation}, we evaluate various convolution kernel 
sizes for the DNA model. 
Large convolution kernel sizes perform better than small
kernel sizes, but increasing kernel size from 7 to 9 does not 
improve the performance.
The standard two-dimensional $7 \times 7$ convolution 
is time-consuming as shown in \tabref{tab:aggregation}, 
because the feature maps in DNA is with the same resolution 
as original images. 
Hence, we use asymmetric convolutions (\ie, $1 \times 7, 7 \times 1$) 
to achieve both large kernel size and fast speed.

\section{Conclusion}
Previous deeply-supervised saliency detection networks use
linear side-output prediction aggregation. 
We theoretically and experimentally demonstrate that linear 
side-output aggregation is suboptimal and worse than nonlinear aggregation.
Based on this observation, we propose the DNA module that aggregates
multi-level side-output features in a nonlinear way.
With a simply modified U-Net, DNA can reach new \sArt under 
various metrics when compared with 16 recent saliency models. 
The proposed network also has less parameters and faster running speed,
which demonstrate the effectiveness of DNA.
In the future, we plan to apply DNA to further improve salient object 
detection and exploit it in other vision tasks that need 
multi-scale and multi-level information.

\bibliographystyle{IEEEtran}
\bibliography{reference}

\end{document}